\documentclass{article}

\usepackage[preprint]{neurips_2026}

\usepackage[utf8]{inputenc}
\usepackage[T1]{fontenc}
\usepackage{hyperref}
\usepackage{url}
\usepackage{booktabs}
\usepackage{amsfonts}
\usepackage{amsmath}
\usepackage{amssymb}
\usepackage{amsthm}
\usepackage{nicefrac}
\usepackage{microtype}
\usepackage{xcolor}
\usepackage{graphicx}
\usepackage{tikz}
\usetikzlibrary{bayesnet}
\usetikzlibrary{shapes.geometric,fit}
\usetikzlibrary{arrows.meta}
\newtheorem{theorem}{Theorem}
\newtheorem{proposition}{Proposition}

\newtheorem{remark}{Remark}
\usepackage[most]{tcolorbox}
\usepackage{enumitem}

\title{Causal Representation Learning for Generalisable Recommendation}

\author{%
Yorgos Felekis\thanks{Equal contribution}\hspace{1mm}\\
University of Warwick \\
\And
Michael O'Riordan\footnotemark[1]\\
Spotify \\
\And
Oriol Corcoll \\
Spotify
\And
Ciar\'{a}n M. Gilligan-Lee \\
Spotify \& University College London\\
}

\begin{document}

\maketitle

\begin{abstract}
Predictive models trained on observational data often fail to generalise to the
distributions they encounter when deployed, especially when the training data is
a product of the system being optimised. Recommender systems are a canonical
example: they are trained on interaction logs confounded by the deployed policy,
past user behaviour, and platform filtering. As a result, the training
distribution differs substantially from the candidate distribution scored at
serving time, a gap that makes offline metrics unreliable predictors of online
performance. We address the distribution shift problem with a method motivated
by causal representation learning (CRL). We propose an information-theoretic
disentanglement criterion and prove that its optimum excludes confounder-dependent
non-causal information from the treatment representation. We then derive a tractable 
variational lower bound that makes the criterion optimisable from finite observational 
data alone. The scope of our method is narrower than that of much of the CRL literature, 
in that we target better generalisation under distribution shift, not full identification
of all latent causal factors. This narrower target is what makes the method
practical, requiring only the existing confounded logs, applying to any standard
supervised model, and adding no inference-time cost. Our headline evaluation is
an A/B test with millions of users on the music streaming platform Spotify, applied
to a production ranker for personalised playlist generation. A CRL variant
matched in capacity to the production baseline performed on par offline but
delivered substantial online gains in listener engagement. 
Complementary evidence on a synthetic dataset with known causal structure and a 
semi-synthetic variant of the public KuaiRand recommendation dataset shows the 
same pattern: similar in-distribution performance to the baseline, gains under 
distribution shift. Across all three settings, adding our causal disentanglement 
objective yields better out-of-distribution generalisation.
\end{abstract}

\section{Introduction}
\label{sec:intro}

Predictive models deployed in feedback loops with the systems that produce their
training data face a basic problem: the distribution they are trained on
diverges from the one they are asked to score at serving time. In this work we
focus on recommender systems, where the (user, item, context) tuples appearing
in training logs are those a deployed policy already chose to surface, and are
further filtered by user engagement. This phenomenon, broadly referred to as
\emph{exposure bias} in the recommendation literature~\cite{chen2023bias},
means that at serving time the same model is asked to score a much broader and
differently distributed candidate set. The gap between training and serving-time
distributions is widely understood to be a primary reason that offline ranking
metrics are often unreliable predictors of online performance. Causal representation learning (CRL)~\cite{scholkopf2021toward} offers a principled solution.
If a model's predictions rely only on causal mechanisms that generate the outcome,
they should remain valid under shifts that leave those mechanisms intact, including the
policy-induced shift between training and serving distributions. 
Classical CRL aims to recover latent causal factors with
identifiability guarantees~\citep{khemakhem2020variational}.
Such identification requires additional assumptions or
information~\citep{locatello2019challenging}, which can be
difficult to justify or obtain from logged observational
recommendation data. A second strand applies CRL to recommendation
directly, including in production.
\citet{wang2022ood} treat user feature drift as an intervention
via a counterfactual VAE that infers latent user state and performs
post-intervention inference at serving time.
\citet{corcoll2026contrastive} block-identify causal treatment
representations via labelled contrastive pairs that stratify on
confounders and outcomes. DCCL~\cite{zhao2023dccl} and
Rec4Ad~\cite{gao2023rec4ad} disentangle specific bias mechanisms
(popularity, sample-selection) with components tied to those
mechanisms, and have been deployed online at Kuaishou and Taobao.

In this work, we assume that the relevant confounders are
observed and present in the training data. In particular, we assume that confounders are encoded in the user, context, and
item features that the deployed policy conditions on. 
In this case, full identification is more than what is needed to address 
the specific failure mode introduced by exposure
bias. The spurious pathway can be suppressed directly by enforcing that the learned representation carries no information about the confounder beyond what the supervised target already requires (Figure~\ref{fig:three_level_hierarchy}). This narrower goal does not recover latent factors or enable counterfactual queries, but it does break the dependence that drives the offline-to-online gap, and it does so using only the existing observational training data. We instantiate this idea by learning a treatment representation that
remains predictive of the outcome conditional on the observed
confounder, while penalising mutual information between the learned
representation and the confounder. The resulting criterion keeps the
outcome-relevant causal content of the treatment and removes the
confounder-dependent non-causal leakage that standard rankers may
otherwise absorb. We make the following contributions:
\begin{itemize}[leftmargin=1em]
    \item We propose an information-theoretic disentanglement criterion and prove that, at the population level, it favors representations that retain the causal content of the treatment and discard confounder-dependent non-causal information. We then derive a tractable variational lower bound that makes the criterion optimisable from finite observational data alone.
  \item We introduce a deployment-ready CRL recipe for robustness
        under distribution shift: a straightforward modification of
        any standard supervised model, trained on the existing
        confounded log, with no inference-time cost and no
        requirement for randomised or interventional data.
  \item We validate the method in three settings of increasing
        realism: a synthetic
        SCM, a semi-synthetic variant of the KuaiRand dataset~\cite{gao2022kuairand}, and a live A/B test on a production personalised playlist generation surface with millions of users on the music streaming platform Spotify. Across
        all three settings, a capacity-matched CRL variant matches
        the baseline on in-distribution metrics and outperforms it
        when evaluation moves off the training distribution.
\end{itemize}

\section{Method}
\label{sec:methods}
\begin{figure}[t]
\centering
  \begin{tikzpicture}[
      scale=0.92,
      transform shape,
      font=\sffamily,
      >={Stealth[round]},
      node distance=1.2cm ]
      \definecolor{microPlaneColor}{RGB}{235, 250, 235}
      \definecolor{mesoPlaneColor}{RGB}{220, 240, 240}
      \definecolor{macroPlaneColor}{RGB}{240, 230, 250}
      \definecolor{causalRed}{RGB}{200, 50, 50}
      \definecolor{spuriousBlue}{RGB}{50, 100, 200}
      \definecolor{structureGray}{gray}{0.6}

      \tikzset{microNode/.style={draw, circle, fill=white,
               inner sep=0.5pt, minimum size=0.5cm, font=\scriptsize}}
      \tikzset{clusterNode/.style={draw, circle, fill=white,
               inner sep=2pt, minimum size=1.1cm,
               font=\small\bfseries, dashed}}
      \tikzset{obsNode/.style={draw, circle, fill=white,
               inner sep=2pt, minimum size=0.9cm, font=\small}}
      \tikzset{abstractionEdge/.style={->, dashed, thick, opacity=0.6}}
      \tikzset{layerLabel/.style={anchor=south,
               font=\bfseries\footnotesize, rotate=63.4}}

      \coordinate (L1_BL) at (-4.0, -2.5);
      \coordinate (L1_BR) at (5.0,  -2.5);
      \coordinate (L1_TR) at (6.5,   0.5);
      \coordinate (L1_TL) at (-2.5,  0.5);
      \fill[microPlaneColor, opacity=0.9]
          (L1_BL) -- (L1_BR) -- (L1_TR) -- (L1_TL) -- cycle;
      \node[layerLabel, green!40!black] at (6.3, -1.0)
          {Micro-Latent Space};

      \node[obsNode] (Z_micro) at (1.5,  -1.2) {$Z$};
      \node[obsNode] (Y_micro) at (-2.5, -1.2) {$Y$};

      \node[microNode, draw=causalRed]   (tc1)  at (-1.7, -0.1) {$t_c^1$};
      \node[microNode, draw=causalRed]   (tc2)  at (-0.9, -0.8) {$t_c^2$};
      \node[microNode, draw=causalRed]   (tck)  at ( 0.1, -1.7) {$t_c^k$};
      \node[microNode, draw=spuriousBlue](tnc1) at ( 2.5, -0.1) {$t_{nc}^1$};
      \node[microNode, draw=spuriousBlue](tnc2) at ( 3.5, -0.8) {$t_{nc}^2$};
      \node[microNode, draw=spuriousBlue](tncd) at ( 4.4, -1.7) {$t_{nc}^d$};

      \foreach \n in {tc1, tck, tnc1, tncd}
          \draw[->, structureGray] (Z_micro) -- (\n);
      \draw[->, structureGray] (Z_micro) -- (Y_micro);
      \foreach \n in {tc1, tc2, tck}
          \draw[->, causalRed, thick] (\n) -- (Y_micro);
      \draw[->, causalRed, thin] (tc1)  to[bend left=20]  (tc2);
      \draw[->, causalRed, thin] (tc2)  to[bend left=20]  (tck);
      \draw[->, spuriousBlue, thin] (tnc1) to[bend right=20] (tnc2);
      \draw[->, spuriousBlue, thin] (tnc2) to[bend right=20] (tncd);

      \def\MesoOffset{2.8}
      \coordinate (L2_BL) at (-4.0, -2.5+\MesoOffset);
      \coordinate (L2_BR) at ( 5.0, -2.5+\MesoOffset);
      \coordinate (L2_TR) at ( 6.5,  0.5+\MesoOffset);
      \coordinate (L2_TL) at (-2.5,  0.5+\MesoOffset);
      \fill[mesoPlaneColor, opacity=0.8]
          (L2_BL) -- (L2_BR) -- (L2_TR) -- (L2_TL) -- cycle;
      \node[layerLabel, teal!60] at (6.3, -1.0+\MesoOffset)
          {Meso-Latent Clusters};

      \node[clusterNode, draw=causalRed]
          (TC)  at (-0.8, -0.8+\MesoOffset) {$T_C$};
      \node[clusterNode, draw=spuriousBlue]
          (TnC) at ( 3.5, -0.8+\MesoOffset) {$T_{nC}$};

      \foreach \n in {tc1, tc2, tck}
          \draw[abstractionEdge, causalRed] (\n.north) -- (TC.south);
      \foreach \n in {tnc1, tnc2, tncd}
          \draw[abstractionEdge, spuriousBlue] (\n.north) -- (TnC.south);

      \def\MacroOffset{5.6}
      \coordinate (L3_BL) at (-4.0, -2.5+\MacroOffset);
      \coordinate (L3_BR) at ( 5.0, -2.5+\MacroOffset);
      \coordinate (L3_TR) at ( 6.5,  0.5+\MacroOffset);
      \coordinate (L3_TL) at (-2.5,  0.5+\MacroOffset);
      \fill[macroPlaneColor, opacity=0.7]
          (L3_BL) -- (L3_BR) -- (L3_TR) -- (L3_TL) -- cycle;
      \node[layerLabel, violet!60] at (6.3, -1.0+\MacroOffset)
          {Macro-Observed Space};

      \node[obsNode] (Z_macro) at ( 4.5, -0.2+\MacroOffset) {$Z$};
      \node[clusterNode, thick]
                     (T_obs)  at ( 1.5, -1.5+\MacroOffset) {$T$};
      \node[obsNode] (Y_macro) at (-1.5, -0.2+\MacroOffset) {$Y$};

      \draw[->] (Z_macro) -- (T_obs);
      \draw[->] (Z_macro) -- (Y_macro);
      \draw[->] (T_obs)   -- (Y_macro);

      \draw[abstractionEdge, causalRed,    very thick]
          (TC.north)  -- (T_obs.south);
      \draw[abstractionEdge, spuriousBlue, very thick]
          (TnC.north) -- (T_obs.south);

      \node[font=\small, align=center, fill=white, inner sep=2pt,
            opacity=0.95, text=black, draw=gray!30, rounded corners]
          at (1.4, \MesoOffset-0.1) {$m(T_C, T_{nC})$};

  \end{tikzpicture}

  \caption{\textbf{Hierarchical Causal Entanglement.}
    \textit{Bottom (Micro-Latent):} Only the red subset $\{t_c^i\}$
    causally drives $Y$; the blue subset $\{t_{nc}^j\}$ contains
    spurious factors correlated with $Z$.
    \textit{Middle (Meso-Latent):} Micro-variables abstract into
    causal ($T_C$) and non-causal ($T_{nC}$) clusters.
    \textit{Top (Observed):} We observe only the entangled treatment
    $T=m(T_C,T_{nC})$, which creates the spurious backdoor path
    $T \leftarrow T_{nC} \leftarrow Z \rightarrow Y$ that standard
    models will exploit.}
  \label{fig:three_level_hierarchy}
\end{figure}
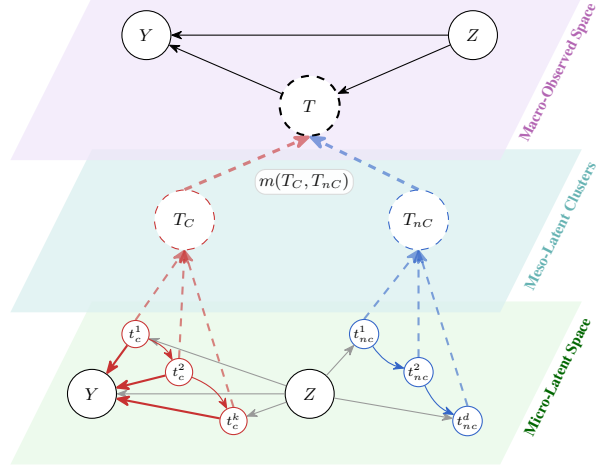
\subsection{Setting}

We adopt the Structural Causal Model framework of~\cite{Pearl2009}.
A causal model is a tuple\(
    \mathcal{M}=(\mathbf{V},\mathbf{U},\mathcal{F},P(\mathbf{U})),\)
where $\mathbf{V}$ is the set of endogenous variables, $\mathbf{U}$
is the set of exogenous noise variables with joint distribution
$P(\mathbf{U})$, and $\mathcal{F}=\{f_V\}_{V\in\mathbf{V}}$ is a
collection of structural assignments
\( V := f_V(\mathrm{Pa}(V),U_V).\) In our setting, $\mathbf{V}=\{T,Y,Z\}$, where $Z$ is an observed
confounder (in recommendation, a user representation), $T$ is the
treatment to be ranked (an item), and $Y$ is the outcome of interest. As illustrated in Figure~\ref{fig:three_level_hierarchy}, the
observed treatment is an entangled mixing
\(T=m(T_C,T_{nC})
\)
of latent causal components $T_C$ and non-causal components $T_{nC}$. We make two structural assumptions: \textbf{(A1)} $Y \perp T_{nC} \mid T_C, Z$ meaning that the outcome depends on the treatment only through its causal latents, given the confounder; \textbf{(A2)} $I(T_{nC}; Z \mid T_C) > 0$: which we call \emph{unique spurious leakage}: the non-causal latents
retain information about the confounder that is not already contained in the causal latents. This is the setting that arises in recommender systems whenever the deployed policy has correlated non-causal item-side factors with user-side signal, creating a spurious backdoor path $T \leftarrow T_{nC} \leftarrow Z \to Y$ that a standard predictor learning $\mathbb{E}[Y \mid T, Z]$ will absorb. Such structures arise also in many causal representation learning settings where spurious correlations between treatment and confounders lead to biased models if not explicitly addressed. 

While maximising $I(T;Y \mid Z)$ ensures the learned representation is predictive of the outcome, it may still preserve spurious signals as it does not penalise representations for carrying extra non-causal information. To address this, we introduce a more targeted objective that trades predictive utility against confounder dependence.

\subsection{A New Objective for Causal Disentanglement}
We propose to learn a representation $g(T)$ that maximizes the following:
\begin{equation}\label{eq:Jg}
\tcboxmath[
  colback=gray!5,
  colframe=gray!55,
  boxrule=0.6pt,
  arc=1.5mm,
  left=2mm,
  right=2mm,
  top=1.5mm,
  bottom=1.5mm
]{
    J(g)
    =
    \underbrace{I(g(T);Y\mid Z)}_{\text{Utility}}
    -
    \lambda\,
    \underbrace{I(g(T);Z)}_{\text{Penalty}}
}
\end{equation}
called the \textit{Disentanglement Criterion}, where $\lambda \ge 0$ is a hyperparameter that controls the trade-off between predictive utility and confounder invariance. This objective balances two forces: the utility term rewards
representations that are predictive of the outcome $Y$ beyond $Z$, thereby encouraging the preservation of $T_C$; the penalty term discourages information shared with $Z$, thereby suppressing $Z$-correlated non-causal factors. 
Note that the penalty constrains the treatment representation $g(T)$, not the predictor’s direct use of $Z$. The predictor still receives both $g(T)$ and $Z$. The objective discourages redundant confounder information in the treatment representation while rewarding information that predicts $Y$ conditional on $Z$.
Under our structural assumptions, the penalty term is exactly the channel through which $T_{nC}$ leaks into the prediction via the spurious path $T_{nC}\leftarrow Z\to Y$. An equivalent bounded reparameterization is given in Appendix~\ref{appendix:reparam}. We first verify that the criterion ranks the ideal causal
representation above the naive entangled one.
\begin{proposition}[Ideal versus naive]
\label{prop:ideal-vs-naive}
Let $g^\star(T)=T_C$ and let
$\tilde g(T)=T=(T_C,T_{nC})$ be a naive representation that retains both latent
components. Under assumptions \textbf{A1} and
\textbf{A2}, for any $\lambda>0$,
\begin{align}
    J(g^\star)>J(\tilde g).
\end{align}
\end{proposition}

\begin{proof}[Proof sketch]
By the chain rule for CMI: $I(T_C,T_{nC};Y\mid Z)
=
I(T_C;Y\mid Z)
+
I(T_{nC};Y\mid T_C,Z)$.
The second term is zero by \textbf{A1}, so the two representations
have equal utility. For the penalty, $I(T_C,T_{nC};Z)
=
I(T_C;Z)+I(T_{nC};Z\mid T_C)$,
and the second term is strictly positive by \textbf{A2}. Hence the
naive representation incurs a strictly larger penalty while achieving
the same utility. Therefore $J(g^\star)>J(\tilde g)$.
Full proof in Appendix~\ref{appendix:proofs}.
\end{proof}
Proposition~\ref{prop:ideal-vs-naive} compares two fixed
representations. We next formalise the more general mechanism through which the
disentanglement criterion removes non-causal information. The observed
treatment \(T=m(T_C,T_{nC})\) may already lose information about
\(T_C\), as does a learned representation \(g(T)\). We therefore define \(\phi(T_C)\) as the information about \(T_C\) recoverable from \(g(T)\) via a deterministic map \(r\), i.e.\ \(r(g(T))=\phi(T_C)\). A \emph{causally purified representation} is then \(\bar g:=r(g)=\phi(T_C)\). Because \(\bar g\) is a deterministic function of \(g\), the information in \(g\) decomposes exactly into the recovered causal content \(\phi(T_C)\) and a residual beyond it. We assume that this residual carries no additional predictive information about \(Y\)
given \(Z\), but remains informative about the confounder, which are the representation-level counterparts of \textbf{A1} and \textbf{A2}:
\(I(g(T);\,Y\mid\phi(T_C),Z)=0,~I(g(T);\,Z\mid\phi(T_C))>0.\)
The following proposition shows that stripping
the residual strictly improves \(J\).
\begin{proposition}[Causal purification]
\label{prop:purification}
Let \(g(T)\) be a representation satisfying the recoverability
assumption, and define its causal purification as
\(\bar g(T):=r(g(T))=\phi(T_C).\) Then, for any \(\lambda>0\),
\begin{align}
    J(\bar g)>J(g).
\end{align}
\end{proposition}

\begin{proof}[Proof sketch]
Since \(\bar g\) is a deterministic function of \(g\), we have \(I(g;Z)=I(\bar g;Z)+I(g;Z\mid \bar g).\)

By assumption, \(I(g;Y\mid Z) = I(\phi(T_C);Y\mid Z) = I(\bar g;Y\mid Z),\) so the utility is preserved. The penalty gap is \(I(g;Z)-I(\bar g;Z)=I(g;Z\mid \phi(T_C))>0,\) which yields \(J(\bar g)-J(g)
=
\lambda I(g;Z\mid \phi(T_C))>0.\)
Full proof in Appendix~\ref{appendix:proofs}.
\end{proof}
\begin{remark} Proposition~\ref{prop:purification} shows that the disentanglement criterion preserves the causal content available in the treatment representation while removing non-causal confounder-dependent information that is not needed for predicting $Y$. Hence, a maximising representation is purified with respect to the leakage
targeted by the criterion and can be written as
$g^\star(T)=\phi(T_C)$.
\end{remark}
However, the form of $\phi$ remains unspecified. Should it preserve all information in $T_C$ (i.e.\ be invertible), or could a lossy compression be preferable? Since $T_C$ is correlated with both $Y$ (good for utility) and $Z$ (costly for the penalty), any lossy $\phi$ reduces the penalty at the cost of some utility, with the trade-off governed by $\lambda$. To formalise this, let $\phi_\text{id}$ denote the identity (uncompressed) and let $\phi\in\mathcal{H}$ denote any non-invertible alternative transformation, where
$\mathcal{H}$ is the class of all non-invertible lossy functions applicable to
$T_C$. For each $\phi\in\mathcal{H}$ we define
\begin{align}
  \Delta U(\phi)
  :=
  I(T_C;Y\mid Z)-I(\phi(T_C);Y\mid Z),
  \qquad
  \Delta P(\phi)
  :=
  I(T_C;Z)-I(\phi(T_C);Z).
\end{align}
where $\Delta U(\phi)$ is the utility lost by compressing the causal component, while $\Delta P(\phi)$ is the reduction in confounding penalty gained by that compression. We assume a \emph{no-free-lunch condition} according to which every nontrivial lossy transformation $\phi\in\mathcal{H}$ under consideration incurs a genuine utility loss and yields a genuine penalty reduction:
\begin{align}\label{eq:nofree}
\Delta U(\phi)>0~\text{and}~\Delta P(\phi)>0.
\end{align}
We further define the \emph{critical penalty weight} as the ratio of the two:
\begin{align}
    \Lambda(\phi)
  :=
  \frac{\Delta U(\phi)}{\Delta P(\phi)}.
\end{align}
\begin{proposition}[Lossless regime]
\label{prop:lossless}
Let $\lambda^\star := \inf_{\phi\in\mathcal{H}}\Lambda(\phi)$.
Under the no-free-lunch condition~\eqref{eq:nofree}, for any
$\lambda\in(0,\lambda^\star)$, the uncompressed
representation $T_C$ strictly dominates every lossy alternative $\phi(T_C)$ with $\phi\in\mathcal{H}$ under $J$; i.e. \(J(T_C)>J(\phi(T_C)),~\forall \phi\in\mathcal{H}.\)
\end{proposition}
\begin{proof}
Fix any $\phi\in\mathcal{H}$. Expanding the definition of $J$ gives
\[
I(T_C;Y\mid Z)-\lambda I(T_C;Z)
>
I(\phi(T_C);Y\mid Z)-\lambda I(\phi(T_C);Z) \iff \lambda
    <
    \frac{\Delta U(\phi)}{\Delta P(\phi)}
    =
    \Lambda(\phi)
\]
Now suppose $\lambda \in (0,\lambda^\star)$. Since $\lambda<\lambda^\star=\inf_{\phi\in\mathcal{H}}\Lambda(\phi)
\le\Lambda(\phi)$ for every $\phi\in\mathcal{H}$, the condition
$\lambda<\Lambda(\phi)$ holds uniformly, so $J(T_C)>J(\phi(T_C))$
for every $\phi\in\mathcal{H}$. Hence, throughout the regime
$(0,\lambda^\star)$, the optimal representation
preserves the full causal component.
\end{proof}
Proposition~\ref{prop:lossless} shows that after non-causal leakage has been removed, sufficiently small penalty weights $\lambda$ preserve the full causal component, and $\phi_\text{id}$ is optimal; compression of $T_C$ becomes preferable only once the penalty reduction outweighs the corresponding utility loss. Finally, we provide a concrete sanity check for the preceding claims in Appendix~\ref{appendix:sanity} for a linear Gaussian SCM where all mutual-information terms can be computed in closed form: the appendix verifies that $J$ strictly prefers the causal representation $T_C$ over the naive entangled representation $(T_C,T_{nC})$, and illustrates the lossless regime by varying the amount of stochastic compression applied to $T_C$.

\subsection{A Tractable Objective For Causal Disentanglement}
\label{sec:estimator}
The criterion in~\eqref{eq:Jg} cannot be optimised directly since both
$I(g(T);Y\mid Z)$ and $I(g(T);Z)$ are unknown and intractable to
compute from finite data. We therefore derive a tractable surrogate by
replacing each term with a computable bound: a lower bound on the
utility and an upper bound on the penalty so that the surrogate is a
valid lower bound on $J(g)$.

\paragraph{Lower bound on the utility.}
For any conditional distribution $q(y\mid g(T),z)$, the standard
variational bound on conditional mutual information (CMI)~\cite{barber2003algorithm}
gives
\begin{align}
  I(g(T);\,Y\mid Z)
  \;\ge\;
  \mathbb{E}_{p(g(T),y,z)}\!\bigl[\log q(y\mid g(T),z)\bigr]
  + H(Y\mid Z),
\end{align}
which motivates learning $q$ to approximate the true posterior
$p(y\mid g(T),z)$. We instantiate this via a contrastive InfoNCE
estimator~\cite{oord2018representation}: for each anchor
$(g(T),y^+,z)$ we draw $K$ negative outcomes
$y^-_1,\ldots,y^-_K$ from $p(y\mid z)$ and define
\begin{align}
  \mathcal{L}_{\mathrm{utility}}(\theta)
  := -\mathbb{E}\!\left[
       \log\frac{e^{f_\theta(g(T),\,y^+)/\tau}}
                {e^{f_\theta(g(T),\,y^+)/\tau}
                 +\sum_{k=1}^K e^{f_\theta(g(T),\,y^-_k)/\tau}}
     \right].
\end{align}
The InfoNCE loss directly lower-bounds the CMI~\cite{oord2018representation}:
\begin{equation}
  I(g(T);\,Y\mid Z)
  \;\ge\;
  \log(K{+}1) - \mathcal{L}_{\mathrm{utility}}(\theta)
  =: -\mathcal{L}_{\mathrm{utility}}(\theta) + C_Y,
  \quad C_Y := \log(K{+}1).
  \label{eq:infonce}
\end{equation}

\paragraph{Upper bound on the penalty.}
For any representation $g(T)$, and optimal InfoNCE critic $d_{\phi^\star}$, the NCE-CLUB
estimator~\cite{liang2023factorizedcontrastivelearninggoing}
upper-bounds the mutual information between $g(T)$ and $Z$:
\begin{equation}\label{eq:nceclub}
  I(g(T);\,Z)
  \;\le\;
  \underbrace{
    \mathbb{E}_{p(g(T),z)}\bigl[d_{\phi^\star}(g(T),z)\bigr]
   -\mathbb{E}_{p(g(T))p(z)}\bigl[d_{\phi^\star}(g(T),z)\bigr]
  }_{=:\,I_{\mathrm{NCE\text{-}CLUB}}(g(T);\,Z;\,d_{\phi^\star})}.
\end{equation}
The first term scores \emph{positive} pairs $(g(T),z)\sim p(g,z)$,
while the second scores \emph{negative} pairs drawn independently from
$p(g)p(z)$. The guarantee assumes an optimal critic recovering the log-density
ratio $d_{\phi^\star}(g,z)=\log p(z\mid g)-\log p(z)$. In practice, we
approximate this with a learnable critic
$d_\phi(g(T),z)$,
trained to discriminate joint samples $(g,z)\sim p(g,z)$ from
independent samples $(g,z)\sim p(g)p(z)$.
Further details are provided in Appendix~\ref{appendix:nceclub_app}.

Substituting~\eqref{eq:infonce} and~\eqref{eq:nceclub} into~\eqref{eq:Jg}, yields a tractable variational lower bound for $J(g)$.

\begin{theorem}
\label{thm:bound}
For any encoder $g_\theta$ and optimal InfoNCE critic $d_{\phi^\star}$,
\begin{align}
  J(g_\theta)
  \;\ge\;
  -\mathcal{L}_{\mathrm{utility}}(\theta)
  -\lambda\, I_{\mathrm{NCE\text{-}CLUB}}(g_\theta(T);\,Z;\,d_{\phi^\star})
  + C_Y.
\end{align}
Since $C_Y=\log(K{+}1)$ does not depend on $(\theta,\phi^\star)$,
maximising this bound is equivalent to minimising
\begin{equation}
  \mathcal{L}_{\mathrm{final}}(\theta,\phi^\star)
  := \mathcal{L}_{\mathrm{utility}}(\theta)
   + \lambda\,\mathcal{L}_{\mathrm{penalty}}(\theta,\phi^\star),
  \quad
  \mathcal{L}_{\mathrm{penalty}}(\theta,\phi^\star)
  := I_{\mathrm{NCE\text{-}CLUB}}(g_\theta(T);\,Z;\,d_{\phi^\star}).
  \label{eq:final-loss}
\end{equation}
\end{theorem}

The resulting objective augments the standard ranking loss with an explicit mutual-information penalty that discourages dependence between the learned representation and the confounder, unlike IPM-based methods~\cite{shalit2017estimating} which require explicit population matching.

\subsection{Practical implementation}
\label{sec:implementation}
\begin{figure}[t]
  \centering
  \includegraphics[width=0.65\linewidth]{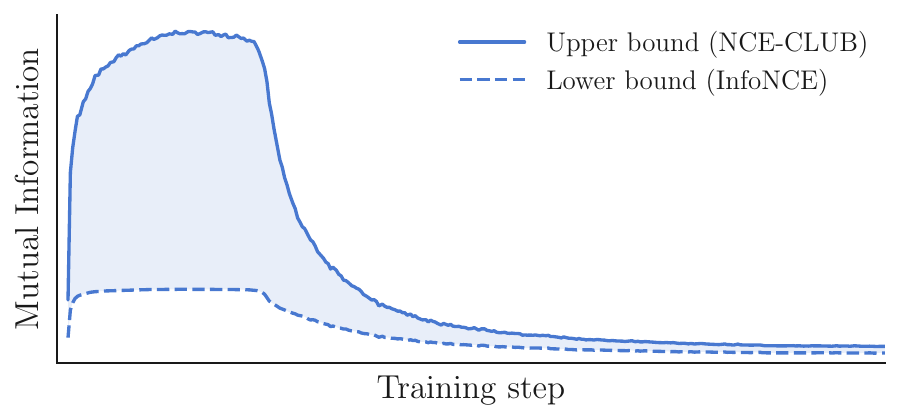}
  \caption{InfoNCE lower bound (dashed) and NCE-CLUB upper bound
    (solid) on $I(g(T); Z)$ during a production CRL training run.
    Gradient reversal acts on the InfoNCE bound. NCE-CLUB is tracked
    only as a diagnostic. Both rise during the critic warm-up, then
    decay as gradient reversal drives MI down. Axis values are omitted for
    confidentiality. As discussed in Section~\ref{sec:implementation}, a conservative reading of the bounds gives a
    reduction in MI of at least a factor of $5$.}
  \label{fig:mi-bounds}
\end{figure}

The tractable surrogate objective in Theorem~\ref{thm:bound} uses NCE-CLUB
to upper-bound the mutual-information penalty.
As described in 
Appendix~\ref{appendix:nceclub_app}, direct optimisation
alternates between fitting the critic with InfoNCE and
updating the encoder against the resulting NCE-CLUB estimate,
together with the supervised utility loss.
In our experiments, we instead find it convenient to apply gradient
reversal~\citep{ganin2015domain} on the InfoNCE
objective~\citep{oord2018representation}.
This allows the encoder and critic to be updated through
a single backward pass, which was simpler to integrate
into our existing production training pipeline.
We emphasise that this choice is motivated by implementation
convenience and is not a requirement of the proposed
disentanglement criterion.

Of course, minimising an InfoNCE lower bound does not
guarantee a decrease in the mutual information.
We therefore track both the InfoNCE lower-bound estimate
and the NCE-CLUB upper-bound estimate during training.
Figure~\ref{fig:mi-bounds} shows a production run in which
both estimates rise during critic warm-up and then
decrease once gradient reversal becomes active.
In particular, the NCE-CLUB estimate decreases even though
it is not directly optimised.
To quantify the MI reduction relative to the non-CRL baseline model,
we compare the final CRL upper bound, $U_{\mathrm{CRL}}$,
with the baseline lower bound, $L_{\mathrm{base}}$.
Writing $I_{\mathrm{CRL}}$ and $I_{\mathrm{base}}$ for
their respective actual MI,
the validity conditions of the bounds give
$I_{\mathrm{CRL}} \leq U_{\mathrm{CRL}}$ and 
$L_{\mathrm{base}} \leq I_{\mathrm{base}}$.
In our production setup, this diagnostic gives
$U_{\mathrm{CRL}} \leq L_{\mathrm{base}}/5$, which implies
$
I_{\mathrm{CRL}}
\leq U_{\mathrm{CRL}}
\leq \tfrac{1}{5}L_{\mathrm{base}}
\leq \tfrac{1}{5}I_{\mathrm{base}}
$.
Therefore, assuming valid bounds, we can conclude that the MI has decreased by at least a factor of 
$5$ relative to the non-CRL model.
If this NCE-CLUB diagnostic does not support a reduction in MI, we recommend direct optimisation of the NCE-CLUB penalty following Appendix~\ref{appendix:nceclub_app}.

An excessively strong penalty can remove information
useful for prediction. Proposition~\ref{prop:lossless}
characterises the ideal lossless regime, but its population
threshold is not directly observable. Moreover, the
scale of $\lambda$ depends on the parameterisation of
the utility and penalty losses, the critic, and the
application. We therefore treat $\lambda$ as a standard
training hyperparameter and suggest using
in-distribution validation performance as a safeguard
against over-regularisation. We recommend increasing $\lambda$ from
zero and selecting the strongest regularisation that
keeps the validation metric within an
application-specific tolerance of the $\lambda=0$
baseline.
The MI diagnostics provide additional guidance on
whether the penalty is adequately reducing confounder dependence.
Consistent with this approach, our CRL and
non-CRL models remain close on in-distribution
metrics, while the CRL models improve under
distribution shift.

\section{Related work}
\label{sec:related}
\paragraph{Treatment-effect estimation with neural networks.}
CFR~\cite{shalit2017estimating} augments TARNet with an
integral-probability-metric penalty that balances covariate
representations across treatment groups.
CEVAE~\cite{louizos2017causal} addresses unobserved confounders with
a deep generative model using proxies. 
See also \cite{jeunen2022disentangling, van2023estimating, fawkes2025hardness} for alternate methods of dealing with unobserved confounders.
Dragonnet~\cite{shi2019dragonnet} adds a
propensity-prediction head to encourage representations sufficient
for treatment-effect identification. 
Deep Twin Networks \cite{vlontzos2023estimating, vlontzos2026incremental, reynaud2022d} adapt neural networks to answer counterfactual questions \cite{richens2020improving, Pearl2009} by ensuring their architecture matches the twin network \cite{Pearl2009, graham2019copy} of the causal structure in question.
Related work studies structured treatments~\cite{kaddour2021causal,harada2021graphite}
and effects of linguistic properties~\cite{pryzant2021causal}.
GraphITE penalises dependence between treatment and individual
representations using HSIC. We instead use an information-theoretic
criterion with a causal-purification analysis.
\citet{corcoll2026contrastive} learn causal treatment
representations using contrastive pairs stratified on confounders
and outcomes, whereas our approach does not require constructing stratified contrastive pairs.

\paragraph{Debiasing causal recommendation.}
\citet{wang2022ood} address shifts in user attributes
with a causal variational autoencoder that predicts preferences
under updated attributes while reducing the influence of outdated
interaction history.
Debiasing approaches include IPS reweighting~\cite{schnabel2016recommendations,joachims2017unbiased},
pseudo-rating-based asymmetric tri-training~\cite{saito2020asymmetric},
instrumental-variable methods such as IViDR~\cite{deng2024ividr},
proxy-based latent-confounder identification in iDCF~\cite{zhang2023idcf},
and causal adjustment~\cite{wang2020causal,zhang2021causal}.
CausE~\cite{bonner2018cause} couples embeddings learned from biased
and unbiased exposure data, while DICE~\cite{zheng2021dice} and
DCCL~\cite{zhao2023dccl} disentangle interest and conformity.
Rec4Ad~\cite{gao2023rec4ad} augments advertising logs with organic
recommendation samples and learns aligned, decorrelated representations.
Our method instead regularises treatment-confounder dependence
within the existing training log.

\paragraph{Adversarial MI estimation and invariance.}
The practical estimator combines gradient
reversal~\cite{ganin2015domain} with the InfoNCE lower
bound~\cite{oord2018representation}. Alternative MI estimators
include MINE~\cite{belghazi2018mine} and the CLUB upper
bound~\cite{cheng2020club};
NCE-CLUB~\cite{liang2023factorizedcontrastivelearninggoing} provides
the upper-bound machinery used in our diagnostics
(Appendix~\ref{appendix:nceclub_app}). MI-based representation
objectives are also at the core of Deep
InfoMax~\cite{hjelm2019deepinfomax}. 
The disentanglement criterion $J(g)$ is complementary to
invariance-based approaches, including invariant causal
prediction~\cite{peters2016causal} and invariant risk
minimisation~\cite{arjovsky2019irm}. These exploit multiple
environments, whereas we penalise dependence on an observed
confounder without requiring an environment partition.

\paragraph{Causal bottlenecks and abstractions.}
Our objective is related to causal information-bottleneck methods,
which retain task-relevant causal information while discarding
irrelevant variation. The Causal Information
Bottleneck~\cite{simoes25a} learns causal variable abstractions by
compressing an input while preserving causal control over a target,
and Structural Causal Bottleneck Models~\cite{bing2026structural}
assume that causal effects between high-dimensional SCM variables
are mediated by lower-dimensional bottlenecks of their causes. These
perspectives connect to causal
abstraction~\cite{rubenstein2017causal}, which seeks coarser variables or models that are interventionally
consistent with the original system, together with methods that learn such maps~\cite{
zennaro23a,felekis2024causal,kekic2024,xia2024neural,
felekis2026distributionallyrobustcausalabstractions}. In the language
of Figure~\ref{fig:three_level_hierarchy}, such methods would ask
whether the entangled $T=m(T_C,T_{nC})$ admits a coarser description,
ideally $T_C$, that preserves interventional consistency. Our setting is adjacent but deliberately
weaker: with $m$ fixed and unobserved,
we learn a map $T \mapsto g(T)$ sufficient for the
recoverable causal content $\phi(T_C)$, while
removing the $Z$-correlated non-causal leakage carried by $T_{nC}$.
We do not identify $T_C$, learn $m$, or verify model-level consistency; we learn a representation that retains the causal information needed for prediction.

\section{Experiments}
\label{sec:experiments}

We evaluate the proposed method in three settings of increasing
realism. The synthetic SCM (\S\ref{sec:exp-synth}) shows that the
method successfully isolates the causal components of the treatment 
representation in a simple setup where the full data generation process is known. 
The semi-synthetic KuaiRand experiment (\S\ref{sec:exp-kuairand}) 
uses real user and video features, with a synthetic nonlinear outcome function
to test robustness to a controlled change in non-causal feature assignment.
The online A/B test
(\S\ref{sec:exp-online}) verifies that this robustness
translates into measurable online gains on a deployed production
model at Spotify.

\subsection{Synthetic SCM}
\label{sec:exp-synth}

We instantiate a linear additive SCM with a confounder
$Z \in \mathbb{R}^{10}$, a treatment $T \in \mathbb{R}^{10}$
decomposable into $d_c = 5$ causal and $5$ non-causal dimensions, and
an outcome $Y$ that depends on $T_C$ and $Z$ but not $T_{nC}$. Full
data-generating equations are given in Appendix~\ref{appendix:synth}. Since $Y$ is
invariant to interventions on $T_{nC}$, a model that has correctly
isolated the causal part of the treatment should also be invariant to
those interventions. We measure the \emph{intervention sensitivity}
$|\hat Y(T) - \hat Y(T')|$ where $T'$ has its non-causal coordinates
resampled, and the \emph{prediction MAE} on $Y$.
In Figure~\ref{fig:synth-noise-sweep} we sweep the outcome noise
$\sigma_Y$ from $0$ to $1$, and show how the baseline predictor 
is sensitive to interventions on the $T_{nC}$
non-causal signal, whereas the CRL variant's intervention sensitivity
sits essentially at zero across the entire range. The MAE
is matched between the two methods.

\begin{figure}[t]
  \centering
  \includegraphics[width=0.95\linewidth]{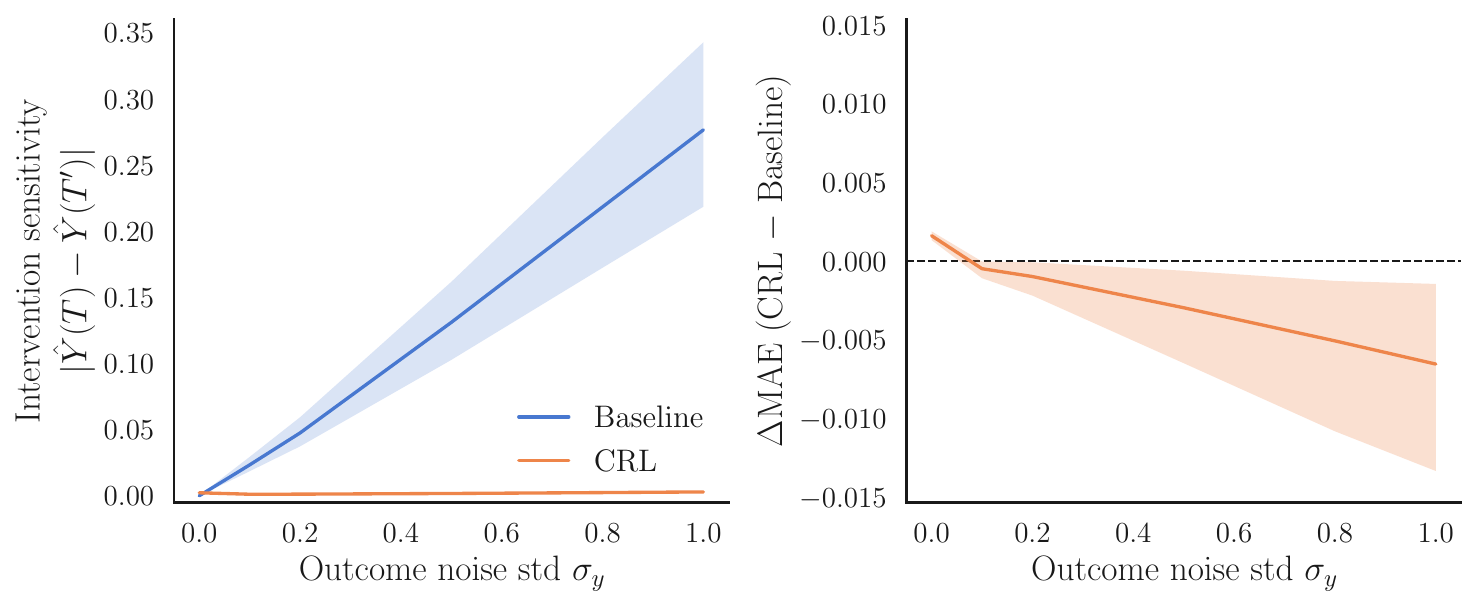}
  \caption{Synthetic SCM, sweep over outcome-noise scale
    $\sigma_Y \in [0, 1]$ across $2{,}000$ seeds per noise level.
    Lines are per-seed medians;
    shaded bands are the inter-quartile range across seeds.
    \textbf{Left:} intervention sensitivity
    $\lvert\hat Y(T) - \hat Y(T')\rvert$, where $T'$ resamples the
    non-causal coordinates of $T$; lower is better, with $0$
    indicating that the model has not absorbed any signal from
    $T_{nC}$. \textbf{Right:} paired
    $\Delta\text{MAE} = \text{MAE}_{\text{CRL}} -
    \text{MAE}_{\text{baseline}}$ per seed, centered on the
    zero line. The CRL penalty does not degrade predictive performance, and is in fact slightly better at
    higher noise levels.}
  \label{fig:synth-noise-sweep}
\end{figure}

\subsection{Semi-Synthetic KuaiRand}
\label{sec:exp-kuairand}

In this section, we evaluate our method in a semi-synthetic setup using
real user and video features from KuaiRand-Pure~\citep{gao2022kuairand}.
We generate synthetic clicks with a fixed nonlinear function of user
features $Z$, video features $T_C$, and their interactions.
We introduce a non-causal feature $T_{nC}$ and append it to the video input.
In the source distribution, this feature is correlated with $Z$, providing a shortcut which the model can exploit for prediction.
In the target distribution, $T_{nC}$ is assigned uniformly at random. 
The source and target evaluations use the same held-out users, videos, and synthetic clicks, with only the shortcut assignment changing.

We use the DeepFM~\citep{guo2017deepfm} architecture with user and video features and learned ID embeddings. 
DeepFM is a standard baseline in the recommendation literature
\citep[see eg.][]{wang2021dcnv2},
and has been applied to the KuaiRand dataset \citep{wang2024gprec,huang2024nise}.
Our CRL variant adds an InfoNCE critic
between the treatment and user representations, with gradient
reversal applied to the treatment representation.
The predictor architecture and capacity are the same as the baseline.
Details of the data-generation, architecture, training setup, and metrics are in Appendix~\ref{appendix:kuairand}.

We compare our CRL approach with the baseline, DCCL~\citep{zhao2023dccl}, and two oracle methods.
DCCL is similar to our approach in that it adds contrastive losses during training without requiring target-domain data. 
However, it targets popularity and conformity through
user-conditioned positive and negative interaction pairs.
Our objective does not specify a particular bias mechanism and does not use user or outcome-specific positive/negative pair construction.
We further note that DCCL's specific bias mechanism and pair construction are too restrictive for our production application.
Oracle IPS uses the exact target-to-source $T_{nC}$ probability ratio. Oracle $T_{nC}$ randomization resamples each training example's $T_{nC}$ uniformly.
Both oracle approaches rely on knowledge of the synthetic assignment mechanism.

In Table~\ref{tab:kuairand}, we show that the baseline is sensitive to
distribution shifts in the non-causal shortcut $T_{nC}$, with AUC degrading between the source and target distributions. DCCL performs similarly to the baseline in this setup. 
Oracle IPS shows little or no degradation between source and target, but large weights and low effective sample size reduce overall performance in this setup. 
The baseline and CRL models are close on the source
distribution, but CRL is robust to the distribution shift, substantially reducing the AUC degradation between source and target. CRL outperforms the baseline across all target metrics, and approaches the performance of the oracle shortcut randomization. This pattern of similar performance on the source distribution combined with gains on the target distribution is similar to what \citet{taori2020measuring}
call \emph{effective robustness}: improvement under shift not
predicted by in-distribution performance alone. In the online A/B test
(Section~\ref{sec:exp-online}) we show how this improved generalisation
ability translates to real gains in listener engagement.

\begin{table}[t]
\centering
\small
\setlength{\tabcolsep}{4pt}
\caption{Semi-synthetic KuaiRand results over $100$ seeds. 
AUC entries report means $\pm$ 95\% confidence intervals. 
Ranking metrics report target means. 
Source and target use the same held-out rows, changing only $T_{nC}$ assignment.
Oracle methods receive information about the true assignment
mechanism. 
Our CRL model improves all target metrics over the baseline
and approaches oracle $T_{nC}$ randomization performance.
}
\label{tab:kuairand}
\begin{tabular}{lcccc}
\toprule
Method & Source AUC & Target AUC & Recall@5 & NDCG@5 \\
\midrule
Baseline & $0.7249 \pm 0.0004$ & $0.6914 \pm 0.0011$ & $0.5697$ & $0.4501$ \\
DCCL & $0.7248 \pm 0.0004$ & $0.6912 \pm 0.0011$ & $0.5699$ & $0.4501$ \\
Oracle IPS & $0.7079 \pm 0.0008$ & $0.7078 \pm 0.0008$ & $0.5850$ & $0.4678$ \\
Oracle $T_{nC}$ randomization & $0.7243 \pm 0.0004$ & $0.7243 \pm 0.0004$ & $0.5884$ & $0.4717$ \\
\midrule
\textbf{CRL (ours)} & $0.7218 \pm 0.0005$ & $0.7212 \pm 0.0004$ & $0.5853$ & $0.4676$ \\
\bottomrule
\end{tabular}
\end{table}

\subsection{Online A/B test on personalised playlist generation}
\label{sec:exp-online}

We evaluate real-world impact through an online A/B test
with millions of users on a production personalised playlist
generation surface at Spotify.
The test compares two versions of the production session-ranking
model: the deployed baseline and a capacity-matched CRL variant
trained with the disentanglement objective of \S\ref{sec:methods}.
The CRL variant uses an adversarial loss to decorrelate the model's
internal representation of the candidate track from its representations
of session- and user-level context, mirroring the recipe used in the
KuaiRand experiment. Both variants perform similarly on in-distribution offline metrics, but the CRL variant performs better online.
Table~\ref{tab:online-ab} reports the headline engagement metrics, 
measured over two weeks during the test.
The CRL variant increases track streams by $+0.75\%$ and minutes
played by $+0.50\%$ while decreasing track skips by $-0.61\%$, with
all three changes significant at the $95\%$ level. For context,
these gains are of the same order of magnitude as the improvements that the current production
ranker, itself a ground-up rebuild, delivered over its mature
predecessor.

\begin{table}[t]
  \centering
  \caption{Online A/B test results on the personalised playlist
    generation surface. Percentage change of the CRL variant relative
    to the production baseline, with 95\% confidence intervals.}
  \label{tab:online-ab}
  \begin{tabular}{lcc}
    \toprule
    Metric & \% change & 95\% CI \\
    \midrule
    Track Streams  & $+0.75\%$ & $[+0.54\%, +0.96\%]$ \\
    Track Skips    & $-0.61\%$ & $[-0.74\%, -0.47\%]$ \\
    Minutes Played & $+0.50\%$ & $[+0.27\%, +0.73\%]$ \\
    \bottomrule
  \end{tabular}
\end{table}

\subsection{Discussion}
\label{sec:exp-discussion}
Across all three settings, in-distribution performance remains
unchanged while performance under distribution shift improves.
The synthetic and semi-synthetic experiments demonstrate robustness to known non-causal treatment variation, and the production A/B test demonstrates online gains in listener engagement. 
Our theoretical analysis assumes that the relevant confounding
information is observed in $Z$. The penalty cannot remove
dependence transmitted only through unobserved confounders.
In production, the available user and session context almost
surely does not include all relevant confounding factors. 
The observed gains
nevertheless suggest that an informative
but incomplete $Z$ can still improve generalisation. Combining the
criterion with proximal causal inference is a possible extension,
requiring suitable proxies and additional identifying assumptions.

\section{Conclusion}
\label{sec:conclusion}
We proposed a CRL-motivated method for making predictive models more robust under the distribution shift characteristic of on-policy recommendation. 
We introduced an information-theoretic disentanglement criterion, characterised its properties, and derived a tractable training objective.
The method uses existing observational logs and adds no
inference-time cost. Synthetic and semi-synthetic experiments
demonstrate improved robustness under controlled shifts,
while a large-scale production A/B test showed significant
gains in listener engagement. 

\begin{ack}
This work was supported by Spotify.
The majority of Yorgos Felekis's contribution was completed
during an internship at Spotify.
Michael O'Riordan, Oriol Corcoll, and Ciar\'an M. Gilligan-Lee
were affiliated with Spotify during this work.
The authors declare no additional competing interests.
\end{ack}

\bibliographystyle{plainnat}
\bibliography{ref}

\newpage
\appendix

\section{Proofs for \S\ref{sec:methods}}
\label{appendix:proofs}

Throughout, we use the structural assumptions introduced in
\S\ref{sec:methods}. The observed treatment decomposes as
\[
    T=m(T_C,T_{nC}),
\]
the outcome is conditionally independent of the non-causal latents,
\[
    \textbf{(A1)}\qquad
    Y\perp T_{nC}\mid(T_C,Z),
\]
and the non-causal latents contain residual confounder information
beyond the causal latents,
\[
    \textbf{(A2)}\qquad
    I(T_{nC};Z\mid T_C)>0.
\]

\subsection*{Proof of Proposition~\ref{prop:ideal-vs-naive}}

We compare the ideal representation \(g^\star(T)=T_C\) with the naive
representation \(\tilde g(T)=T=(T_C,T_{nC})\).

\paragraph{Utility.}
By the chain rule for conditional mutual information,
\[
\begin{aligned}
I(\tilde g(T);Y\mid Z)
&= I(T_C,T_{nC};Y\mid Z) \\
&= I(T_C;Y\mid Z)
   + I(T_{nC};Y\mid T_C,Z).
\end{aligned}
\]
The second term is zero by \textbf{A1}. Hence
\[
    I(\tilde g(T);Y\mid Z)
    =
    I(T_C;Y\mid Z)
    =
    I(g^\star(T);Y\mid Z).
\]

\paragraph{Penalty.}
Again, by the mutual-information chain rule,
\[
\begin{aligned}
I(\tilde g(T);Z)
&= I(T_C,T_{nC};Z) \\
&= I(T_C;Z)+I(T_{nC};Z\mid T_C).
\end{aligned}
\]
By \textbf{A2}, \(I(T_{nC};Z\mid T_C)>0\). Therefore
\[
    I(\tilde g(T);Z)>I(T_C;Z)=I(g^\star(T);Z).
\]

Equivalently, using the entropy chain rule,
\[
\begin{aligned}
I(T_C,T_{nC};Z)
&=
H(T_C,T_{nC})-H(T_C,T_{nC}\mid Z) \\
&=
\bigl[H(T_C)+H(T_{nC}\mid T_C)\bigr]
-
\bigl[H(T_C\mid Z)+H(T_{nC}\mid T_C,Z)\bigr] \\
&=
I(T_C;Z)+I(T_{nC};Z\mid T_C).
\end{aligned}
\]
Thus, the excess penalty of the naive representation is exactly the
unique spurious leakage term.

\paragraph{Conclusion.}
The two representations have equal utility, but the naive
representation has a strictly larger penalty. Hence
\[
\begin{aligned}
J(g^\star)-J(\tilde g)
&=
\bigl[I(g^\star;Y\mid Z)-I(\tilde g;Y\mid Z)\bigr]
+
\lambda\bigl[I(\tilde g;Z)-I(g^\star;Z)\bigr] \\
&=
\lambda\,I(T_{nC};Z\mid T_C)>0.
\end{aligned}
\]
Therefore \(J(g^\star)>J(\tilde g)\).
\qed

\subsection*{Proof of Proposition~\ref{prop:purification}}

Consider a representation satisfying the recoverability assumption:
\[
    g(T)=h\bigl(\phi(T_C),\ell(T_{nC})\bigr),
\]
where the causal component \(\phi(T_C)\) is deterministically
recoverable from \(g(T)\). That is, there exists a map \(r\) such that
\[
    r(g(T))=\phi(T_C).
\]
Define the purified representation by
\[
    \bar g(T):=r(g(T))=\phi(T_C).
\]
Thus \(\bar g\) is a deterministic function of \(g\). At the representation level, we assume that the residual information
in \(g(T)\) beyond the recovered causal component contains no
additional outcome information but does contain residual confounder
information:
\[
    \textbf{(R1)}\qquad
    I(g(T);Y\mid \phi(T_C),Z)=0,
\]
and
\[
    \textbf{(R2)}\qquad
    I(g(T);Z\mid \phi(T_C))>0.
\]

\paragraph{Utility preservation.}
Since \(\bar g=\phi(T_C)\) is a deterministic function of \(g\), the
pair \((g,\bar g)\) contains the same information as \(g\). Hence
\[
    I(g;Y\mid Z)=I(g,\bar g;Y\mid Z).
\]
By the chain rule for conditional mutual information,
\[
\begin{aligned}
I(g,\bar g;Y\mid Z)
&=
I(\bar g;Y\mid Z)
+
I(g;Y\mid \bar g,Z).
\end{aligned}
\]
Since \(\bar g=\phi(T_C)\), the second term is
\[
    I(g;Y\mid \bar g,Z)
    =
    I(g;Y\mid \phi(T_C),Z),
\]
which is zero by \textbf{R1}. Therefore
\[
    I(g;Y\mid Z)=I(\bar g;Y\mid Z).
\]
Thus purification preserves the utility term.

\paragraph{Penalty reduction.}
Again, since \(\bar g\) is a deterministic function of \(g\),
\[
    I(g;Z)=I(g,\bar g;Z).
\]
Applying the mutual-information chain rule,
\[
\begin{aligned}
I(g,\bar g;Z)
&=
I(\bar g;Z)+I(g;Z\mid \bar g).
\end{aligned}
\]
Using \(\bar g=\phi(T_C)\), we obtain
\[
    I(g;Z)-I(\bar g;Z)
    =
    I(g;Z\mid \bar g)
    =
    I(g;Z\mid \phi(T_C)).
\]
By \textbf{R2}, this quantity is strictly positive. Hence
\[
    I(\bar g;Z)<I(g;Z).
\]

\paragraph{Objective comparison.}
Combining utility preservation and penalty reduction,
\[
\begin{aligned}
J(\bar g)-J(g)
&=
\bigl[I(\bar g;Y\mid Z)-I(g;Y\mid Z)\bigr]
+
\lambda\bigl[I(g;Z)-I(\bar g;Z)\bigr] \\
&=
\lambda\,I(g;Z\mid \phi(T_C))>0.
\end{aligned}
\]
Therefore \(J(\bar g)>J(g)\) for every \(\lambda>0\).
\qed

\subsection*{General Predictive Purification Result}

The purification result in the main text assumes that the causal
component \(\phi(T_C)\) is recoverable from the representation. We
now show that a more general purification principle holds without this
assumption.

Let
\[
    G := g(T)
\]
be an arbitrary representation. We call \(\bar G=s(G)\) a
\emph{predictive purification} of \(G\) if it is a deterministic
function of \(G\) satisfying
\[
    Y\perp G\mid(\bar G,Z).
\]
In words, once \(\bar G\) and \(Z\) are known, the remaining
information in \(G\) does not provide additional predictive
information about \(Y\).

\paragraph{Existence.}
Under standard regularity conditions, such a statistic exists. One
canonical construction is the conditional response profile
\[
    s(G)
    :=
    \bigl\{p(Y\mid G,Z=z):z\in\mathcal{Z}\bigr\}.
\]
This object records, for each value of \(G\), the conditional law of
\(Y\) across values of the confounder. By construction,
\[
    p(Y\mid G,Z)=p(Y\mid s(G),Z),
\]
which implies
\[
    Y\perp G\mid(s(G),Z).
\]
This statistic may be infinite-dimensional and is used only as a
population-level object.

\begin{proposition}[General predictive purification]
\label{prop:purification_general}
Let \(G=g(T)\) be any representation, and let
\(\bar G=s(G)\) be a predictive purification. Then, for any
\(\lambda>0\),
\[
    J(\bar G)\ge J(G),
\]
with exact improvement
\[
    J(\bar G)-J(G)
    =
    \lambda I(G;Z\mid \bar G).
\]
The inequality is strict whenever
\(I(G;Z\mid \bar G)>0\).
\end{proposition}

\begin{proof}
Since \(\bar G=s(G)\) is a deterministic function of \(G\),
\[
    I(G;Y\mid Z)=I(G,\bar G;Y\mid Z).
\]
By the chain rule,
\[
I(G,\bar G;Y\mid Z)
=
I(\bar G;Y\mid Z)
+
I(G;Y\mid \bar G,Z).
\]
The second term is zero by predictive sufficiency, so
\[
    I(G;Y\mid Z)=I(\bar G;Y\mid Z).
\]

For the penalty,
\[
    I(G;Z)=I(G,\bar G;Z)
    =
    I(\bar G;Z)+I(G;Z\mid \bar G).
\]
Thus
\[
    I(G;Z)-I(\bar G;Z)=I(G;Z\mid \bar G)\ge 0.
\]

Combining,
\[
    J(\bar G)-J(G)
    =
    \lambda I(G;Z\mid \bar G)\ge 0.
\]
\end{proof}

\paragraph{Connection to the causal setting.}
The result above is purely information-theoretic and applies to any
representation. Its causal interpretation follows from assumption
\textbf{A1}. For the full entangled latent representation
\[
    G=(T_C,T_{nC}),
\]
we have
\[
    p(Y\mid T_C,T_{nC},Z)=p(Y\mid T_C,Z),
\]
so the predictive purification depends only on the causal component
\(T_C\), up to outcome-equivalence. Thus, in the SCM considered in
this work, the population target of purification can be written as
\[
    g(T)=\phi(T_C).
\]

\paragraph{Relation to the main result.}
Proposition~\ref{prop:purification_general} shows that the
disentanglement criterion removes information that is irrelevant for
predicting \(Y\mid Z\) but still informative about the confounder,
without requiring any structural assumptions on \(g\). The main-text
result strengthens this conclusion by imposing a recoverability
assumption, which allows the purified representation to be identified
explicitly as a function of the causal component \(T_C\).

\section{Reparameterisation of the Disentanglement Criterion}
\label{appendix:reparam}

The objective in~\eqref{eq:Jg},
\[
    J(g)=I(g(T);Y\mid Z)-\lambda I(g(T);Z),
\]
uses the unbounded penalty weight \(\lambda\in[0,\infty)\). It is often
more convenient to work with an equivalent bounded parameter
\(\gamma\in[0,1]\).

\paragraph{Reparameterisation.}
For \(\lambda>0\), maximising \(J(g)\) is equivalent to minimising
\[
    -J(g)=\lambda I(g(T);Z)-I(g(T);Y\mid Z).
\]
Dividing by the positive constant \(1+\lambda\) does not change the
minimiser, so the same solutions are obtained by minimising
\[
    \frac{\lambda}{1+\lambda}I(g(T);Z)
    -
    \frac{1}{1+\lambda}I(g(T);Y\mid Z).
\]
Define
\[
    \gamma:=\frac{1}{1+\lambda}\in(0,1),
\]
so that \(1-\gamma=\lambda/(1+\lambda)\). The equivalent bounded
objective is
\begin{equation}
  L_\gamma(g)
  :=
  (1-\gamma)I(g(T);Z)
  -
  \gamma I(g(T);Y\mid Z),
  \label{eq:Jprime}
\end{equation}
which we minimise over \(g\). The limiting case \(\lambda=0\)
corresponds to \(\gamma=1\), i.e. pure utility maximisation. Conversely,
\(\lambda\to\infty\) corresponds to \(\gamma\to 0\), i.e. pure
penalty minimisation.

\paragraph{Optimality condition under \(\gamma\).}
Using the same notation as Proposition~\ref{prop:lossless}, compare
the uncompressed representation \(T_C\) with a lossy alternative
\(\phi(T_C)\). The uncompressed representation is preferred under
\(L_\gamma\) when
\[
    L_\gamma(T_C)<L_\gamma(\phi(T_C)).
\]
Expanding this inequality gives
\[
    (1-\gamma)\Delta P(\phi)
    -
    \gamma\Delta U(\phi)
    <0,
\]
or equivalently
\[
    \gamma
    >
    \frac{\Delta P(\phi)}
         {\Delta U(\phi)+\Delta P(\phi)}
    =:
    \Gamma(\phi).
\]
Thus \(\Gamma(\phi)\) is the critical utility weight above which the
full causal component is preferred to the lossy compression \(\phi\).

\begin{proposition}[High-utility regime under \(\gamma\)]
\label{prop:lossless-gamma}
Let \(\mathcal{H}\), \(\Delta U(\phi)\), and \(\Delta P(\phi)\) be as
in Proposition~\ref{prop:lossless}, and assume the no-free-lunch
condition \(\Delta U(\phi)>0\) and \(\Delta P(\phi)>0\) for every
\(\phi\in\mathcal{H}\). Define
\[
    \gamma^\star
    :=
    \sup_{\phi\in\mathcal{H}}
    \Gamma(\phi)
    =
    \sup_{\phi\in\mathcal{H}}
    \frac{\Delta P(\phi)}
         {\Delta U(\phi)+\Delta P(\phi)}.
\]
Then for any \(\gamma\in(\gamma^\star,1)\), the uncompressed
representation \(T_C\) strictly dominates every lossy alternative
\(\phi(T_C)\), \(\phi\in\mathcal{H}\), under the objective
\(L_\gamma\).
\end{proposition}

\begin{proof}
Fix any \(\phi\in\mathcal{H}\). The condition
\(L_\gamma(T_C)<L_\gamma(\phi(T_C))\) holds if and only if
\(\gamma>\Gamma(\phi)\). By definition of the supremum,
\[
    \Gamma(\phi)\le\gamma^\star
    \qquad
    \forall\,\phi\in\mathcal{H}.
\]
Hence, for any \(\gamma>\gamma^\star\),
\[
    \gamma>\Gamma(\phi)
    \qquad
    \forall\,\phi\in\mathcal{H}.
\]
Therefore \(L_\gamma(T_C)<L_\gamma(\phi(T_C))\) uniformly over
\(\mathcal{H}\), so the uncompressed causal representation strictly
dominates all lossy alternatives.
\qed
\end{proof}

\paragraph{Correspondence with Proposition~\ref{prop:lossless}.}
The \(\lambda\)- and \(\gamma\)-parameterisations describe the same
trade-off. Since
\[
    \gamma=\frac{1}{1+\lambda},
\]
the condition
\[
    \lambda<\Lambda(\phi)
    =
    \frac{\Delta U(\phi)}{\Delta P(\phi)}
\]
is equivalent to
\[
    \gamma>
    \frac{\Delta P(\phi)}
         {\Delta U(\phi)+\Delta P(\phi)}
    =
    \Gamma(\phi).
\]
Thus the low-\(\lambda\) lossless regime in
Proposition~\ref{prop:lossless} is the same as the high-\(\gamma\)
regime in Proposition~\ref{prop:lossless-gamma}.

\section{Toy Example}
\label{appendix:sanity}

We give a closed-form sanity check for the Disentanglement Criterion: 
\[
    J(g)=I(g(T);Y\mid Z)-\lambda I(g(T);Z)
\]
in a linear additive Gaussian model. The treatment consists of causal and non-causal components, but only the causal ones affects the outcome. Consider the following SCM:
\begin{align}
    Z &= \epsilon_Z, \\
    T_C &= \alpha Z + \epsilon_C, \\
    T_{nC} &= \beta Z + \epsilon_N, \\
    Y &= \rho T_C + \delta Z + \epsilon_Y,
\end{align}
where all noise variables are mutually independent and Gaussian:
\[
    \epsilon_Z\sim\mathcal{N}(0,\sigma_Z^2),\quad
    \epsilon_C\sim\mathcal{N}(0,\sigma_C^2),\quad
    \epsilon_N\sim\mathcal{N}(0,\sigma_N^2),\quad
    \epsilon_Y\sim\mathcal{N}(0,\sigma_Y^2).
\]
The observed treatment is
\[
    T=(T_C,T_{nC}).
\]
By construction, the outcome depends on the treatment only through the
causal component:
\[
    Y\perp T_{nC}\mid (T_C,Z).
\]
Moreover, under non-degenerate parameters, e.g. $\beta\neq0$,
$\sigma_Z^2>0$, $\sigma_C^2>0$, and $\sigma_N^2>0$, the non-causal
component contains residual information about the confounder beyond
the causal component:
\[
    I(T_{nC};Z\mid T_C)>0.
\]
Indeed, since the variables are jointly Gaussian,
\[
    I(T_{nC};Z\mid T_C)
    =
    \frac12
    \log
    \frac{\mathrm{Var}(T_{nC}\mid T_C)}
         {\mathrm{Var}(T_{nC}\mid Z,T_C)}.
\]
Now
\[
    \mathrm{Var}(T_{nC}\mid Z,T_C)=\sigma_N^2,
\]
whereas
\[
    \mathrm{Var}(T_{nC}\mid T_C)
    =
    \beta^2\mathrm{Var}(Z\mid T_C)+\sigma_N^2,
\]
with
\[
    \mathrm{Var}(Z\mid T_C)
    =
    \frac{\sigma_Z^2\sigma_C^2}
         {\alpha^2\sigma_Z^2+\sigma_C^2}.
\]
Therefore,
\[
    I(T_{nC};Z\mid T_C)
    =
    \frac12
    \log\left(
        1+
        \frac{\beta^2\sigma_Z^2\sigma_C^2}
             {\sigma_N^2(\alpha^2\sigma_Z^2+\sigma_C^2)}
    \right)>0.
\]
Thus the Gaussian SCM satisfies assumptions \textbf{A1} and
\textbf{A2} from the main text. Consequently, we now compare the ideal causal representation with the naive entangled representation:
\[
    g_1(T)=T_C,~~~g_2(T)=T=(T_C,T_{nC}).
\]
For scalar jointly Gaussian variables,
\[
    I(U;V\mid W)
    =
    \frac12\log
    \frac{\mathrm{Var}(V\mid W)}
         {\mathrm{Var}(V\mid U,W)}.
\]
Conditioned on $Z$,
\[
    Y=(\rho\alpha+\delta)Z+\rho\epsilon_C+\epsilon_Y,
\]
so
\[
    \mathrm{Var}(Y\mid Z)=\rho^2\sigma_C^2+\sigma_Y^2.
\]
Conditioned on $(T_C,Z)$,
\[
    Y=\rho T_C+\delta Z+\epsilon_Y,
\]
and therefore
\[
    \mathrm{Var}(Y\mid T_C,Z)=\sigma_Y^2.
\]
Hence
\begin{align}
    I(T_C;Y\mid Z)
    &=
    \frac12
    \log
    \left(
        \frac{\rho^2\sigma_C^2+\sigma_Y^2}{\sigma_Y^2}
    \right) \\
    &=
    \frac12
    \log
    \left(
        1+\frac{\rho^2\sigma_C^2}{\sigma_Y^2}
    \right).
\end{align}
Since \( Y\perp T_{nC}\mid(T_C,Z) \implies I(T_C,T_{nC};Y\mid Z)=I(T_C;Y\mid Z)\). So the ideal and naive representations have identical utility.

For the penalty term,
\[
    I(T_C;Z)
    =
    \frac12
    \log
    \left(
        \frac{\alpha^2\sigma_Z^2+\sigma_C^2}{\sigma_C^2}
    \right)
    =
    \frac12
    \log
    \left(
        1+\frac{\alpha^2\sigma_Z^2}{\sigma_C^2}
    \right).
\]
By the mutual-information chain rule,
\[
    I(T_C,T_{nC};Z)
    =
    I(T_C;Z)+I(T_{nC};Z\mid T_C).
\]
The second term is strictly positive by the calculation above, so
\[
    I(g_2(T);Z)>I(g_1(T);Z).
\]
Therefore, for any $\lambda>0$,
\begin{align}
    J(g_1)-J(g_2)
    &=
    \lambda
    \bigl[
        I(g_2(T);Z)-I(g_1(T);Z)
    \bigr] \\
    &=
    \lambda I(T_{nC};Z\mid T_C)>0.
\end{align}
Thus, in this Gaussian SCM, the criterion strictly prefers the causal
representation $T_C$ over the naive entangled representation
$(T_C,T_{nC})$. Figure~\ref{fig:jg_lambda_plot} shows the same effect
empirically in a multivariate simulation where the representation
progressively removes coordinates of $T_{nC}$. Retaining more
non-causal information leaves the utility unchanged but increases the
confounding penalty, so the corresponding objective values decrease
more rapidly as $\lambda$ grows.

\begin{figure}[htb]
    \centering
    \includegraphics[width=0.8\linewidth]{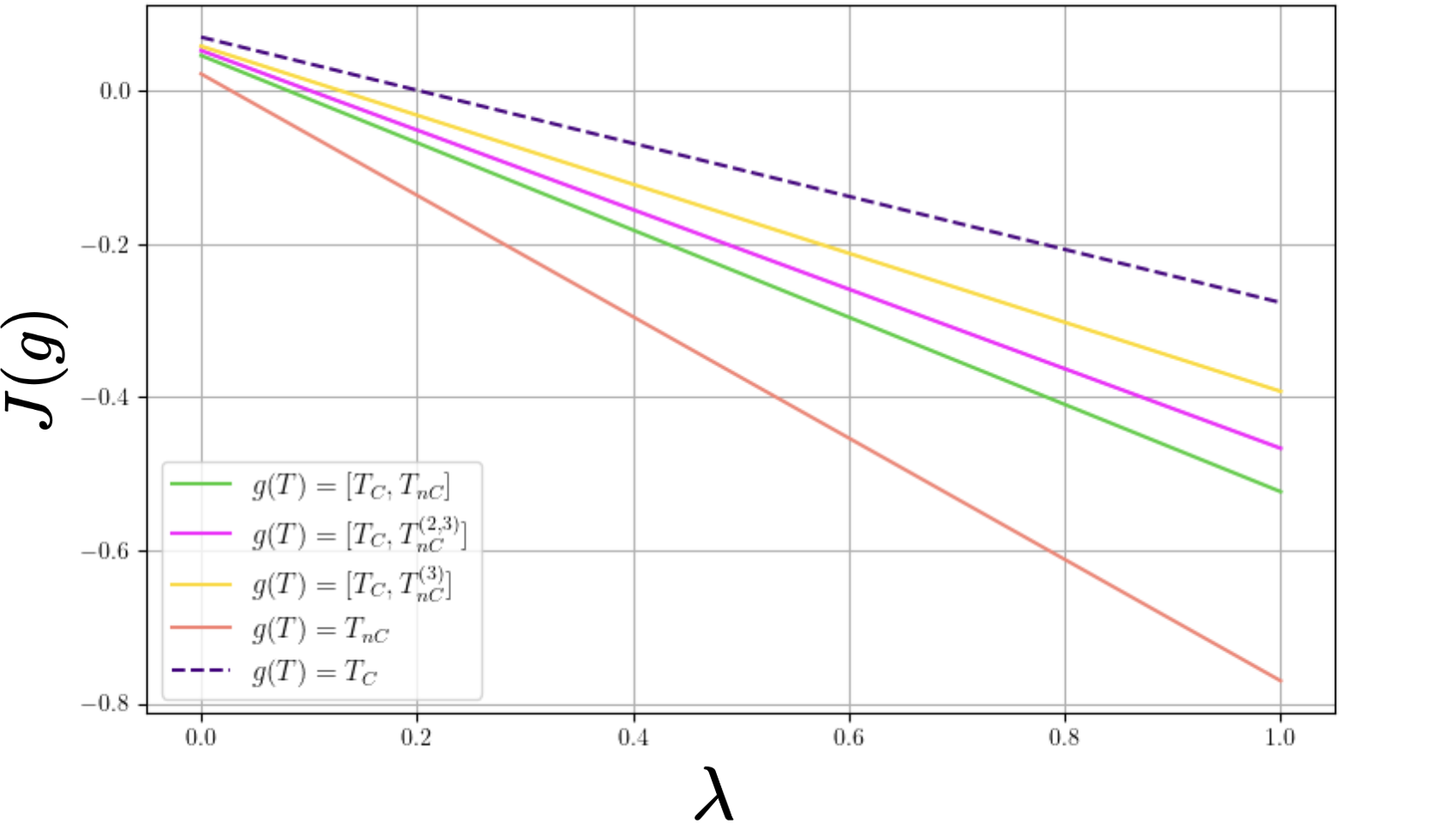}
    \caption{\textbf{Progressive removal of non-causal leakage.}
    The objective $J(g)$ is plotted as a function of the penalty
    weight $\lambda$ for representations retaining different amounts
    of the non-causal component $T_{nC}$. Since $T_{nC}$ contributes
    confounder information without adding outcome information beyond
    $T_C$, representations with more non-causal leakage incur a larger
    penalty. The purely causal representation $g(T)=T_C$ achieves the
    highest score among the tested representations.}
    \label{fig:jg_lambda_plot}
\end{figure}
\begin{figure}[ht!]
    \centering
    \includegraphics[width=0.8\linewidth]{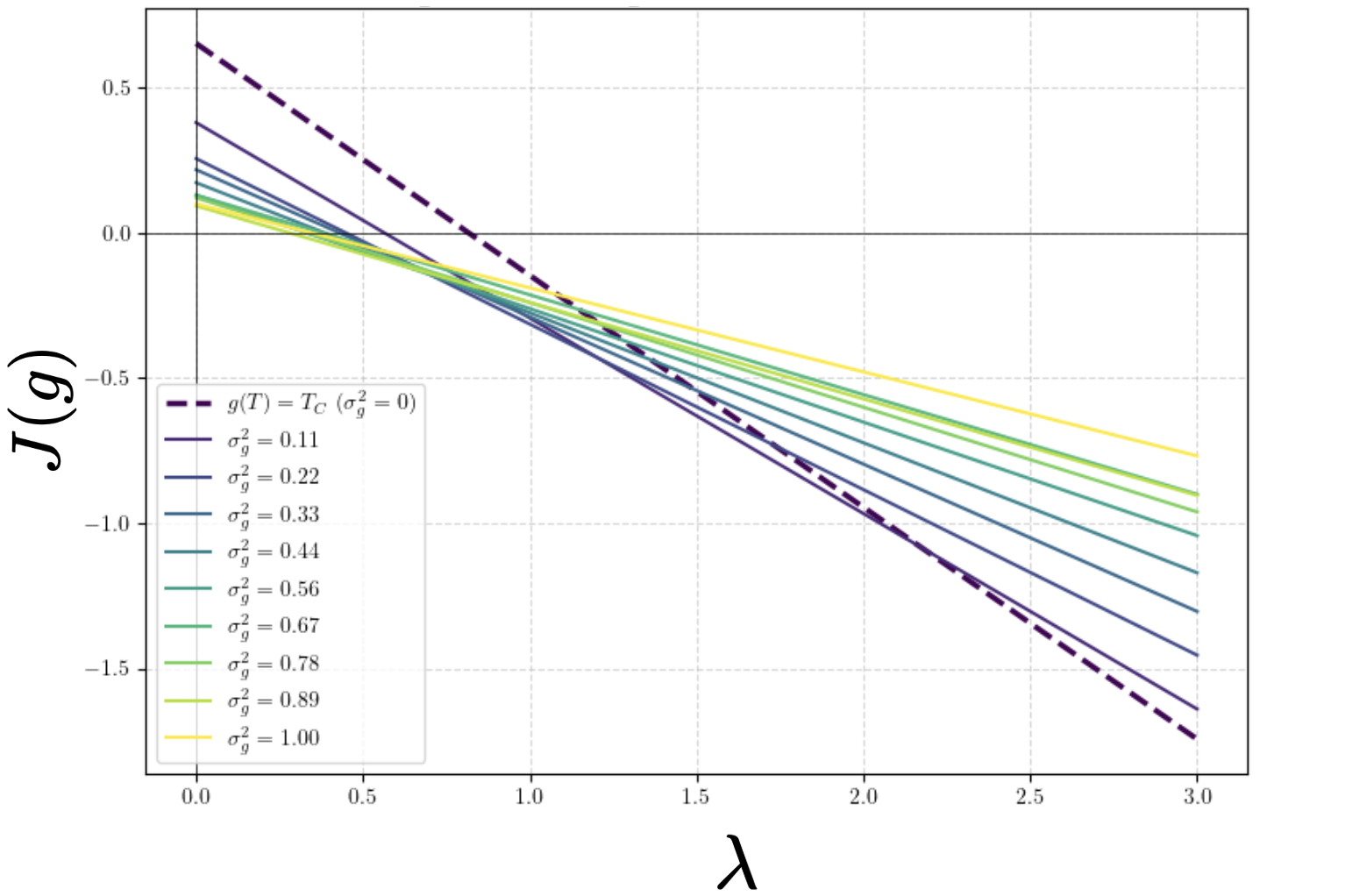}
    \caption{\textbf{Lossless regime and critical penalty weight.}
    The objective $J(g)$ is plotted as a function of $\lambda$ for
    stochastic compressed representations
    $G=T_C+\epsilon_G$, where
    $\epsilon_G\sim\mathcal{N}(0,\sigma_G^2)$. The dashed curve
    corresponds to the uncompressed causal representation
    $\sigma_G^2=0$. For small $\lambda$, preserving the full causal
    component is optimal; for larger $\lambda$, compression can become
    preferable because the reduction in confounder information
    outweighs the utility loss.}
    \label{fig:lambda_crit_demo}
\end{figure}
After the non-causal component has been removed, the remaining question
is whether the optimal representation should preserve all of $T_C$ or
compress it. To model lossy compression analytically, consider the
stochastic representation channel
\[
    G = T_C+\epsilon_G,
    \qquad
    \epsilon_G\sim\mathcal{N}(0,\sigma_G^2),
\]
with $\epsilon_G$ independent of all other variables. This can be viewed as a Gaussian Markov kernel from $T_C$ to $G$, where $\sigma_G^2$ controls the amount of compression. The penalty is
\[
    P(\sigma_G^2)
    :=
    I(G;Z)
    =
    \frac12
    \log\left(
        1+
        \frac{\alpha^2\sigma_Z^2}{\sigma_C^2+\sigma_G^2}
    \right),
\]
and the utility is
\[
    U(\sigma_G^2)
    :=
    I(G;Y\mid Z)
    =
    \frac12
    \log\left(
        \frac{\rho^2\sigma_C^2+\sigma_Y^2}
        {
        (\rho^2\sigma_C^2+\sigma_Y^2)
        -
        \frac{\rho^2\sigma_C^4}{\sigma_C^2+\sigma_G^2}
        }
    \right).
\]
Both terms decrease as $\sigma_G^2$ increases: compression removes
information about the confounder, but it also removes predictive
information about the outcome. The objective is
\[
    J(\sigma_G^2)
    =
    U(\sigma_G^2)-\lambda P(\sigma_G^2).
\]
At the uncompressed boundary $\sigma_G^2=0$,
\[
    \frac{dP}{d\sigma_G^2}\bigg|_{\sigma_G^2=0}
    =
    -\frac{\alpha^2\sigma_Z^2}
    {2\sigma_C^2(\sigma_C^2+\alpha^2\sigma_Z^2)},
\]
and
\[
    \frac{dU}{d\sigma_G^2}\bigg|_{\sigma_G^2=0}
    =
    -\frac{\rho^2}{2\sigma_Y^2}.
\]
Thus the uncompressed representation is locally optimal whenever
\[
    \frac{dJ}{d\sigma_G^2}\bigg|_{\sigma_G^2=0}
    =
    \frac{dU}{d\sigma_G^2}\bigg|_{\sigma_G^2=0}
    -
    \lambda
    \frac{dP}{d\sigma_G^2}\bigg|_{\sigma_G^2=0}
    \le 0.
\]
Equivalently,
\[
    \lambda
    \le
    \lambda_{\mathrm{crit}}
    :=
    \frac{
    \left|
    \frac{dU}{d\sigma_G^2}
    \right|_{\sigma_G^2=0}
    }{
    \left|
    \frac{dP}{d\sigma_G^2}
    \right|_{\sigma_G^2=0}
    }.
\]
This threshold is the ratio between the marginal utility loss and the
marginal penalty reduction at zero compression. Figure~\ref{fig:lambda_crit_demo}
visualises this trade-off: for small $\lambda$, preserving the full
causal component is optimal; for larger $\lambda$, the penalty becomes
sufficiently expensive that a noisy compressed representation can
achieve a higher objective value. This is the Gaussian analogue of the
lossless regime in Proposition~\ref{prop:lossless}.

\section{NCE-CLUB Upper Bound and Critic Interpretation}
\label{appendix:nceclub_app}

To construct a tractable surrogate for the disentanglement objective
\[
    J(g)=I(g(T);Y\mid Z)-\lambda I(g(T);Z),
\]
we require an upper bound on the penalty term \(I(g(T);Z)\).

\paragraph{General NCE-CLUB bound.}
We use the NCE-CLUB estimator~\cite{liang2023factorizedcontrastivelearninggoing}.
For a pair of random variables \((X_1,X_2)\), NCE-CLUB provides the
upper bound
\[
    I(X_1;X_2)
    \le
    I_{\mathrm{NCE\text{-}CLUB}}(X_1;X_2),
\]
where, for the optimal critic \(f^\star\),
\[
\begin{aligned}
I_{\mathrm{NCE\text{-}CLUB}}(X_1;X_2)
:=
\mathbb{E}_{(x_1,x_2^+)\sim p(x_1,x_2)}
    \bigl[f^\star(x_1,x_2^+)\bigr]  -
\mathbb{E}_{x_1\sim p(x_1),\,x_2^-\sim p(x_2)}
    \bigl[f^\star(x_1,x_2^-)\bigr].
\end{aligned}
\]
The first expectation is over joint, or positive, pairs; the second is
over independently sampled, or negative, pairs.

\paragraph{Application to the penalty term.}
Applying this construction to the penalty with
\[
    X_1=g(T),\qquad X_2=Z,
\]
gives
\[
    I(g(T);Z)
    \le
    I_{\mathrm{NCE\text{-}CLUB}}(g(T);Z).
\]
In practice, the optimal critic \(f^\star\) is unknown, so we use a
learned critic \(d_\phi\) and obtain the plug-in estimator
\[
\begin{aligned}
I_{\mathrm{NCE\text{-}CLUB}}(g(T);Z;d_\phi)
&:=
\mathbb{E}_{p(g(T),z)}
    \bigl[d_\phi(g(T),z)\bigr] \\
&\quad -
\mathbb{E}_{p(g(T))p(z)}
    \bigl[d_\phi(g(T),z)\bigr].
\end{aligned}
\]
This is the quantity used as the penalty estimator in the main text.

\paragraph{Interpretation of the critic \(d_\phi\).}
The optimal critic corresponds to the log-density ratio
\[
    d^\star(g,z)
    =
    \log\frac{p(z\mid g)}{p(z)}
    =
    \log p(z\mid g)-\log p(z),
\]
which is intractable because the true conditional \(p(z\mid g)\) is
unknown. We therefore approximate it with a learnable critic
\(d_\phi(g,z)\), which can be interpreted as a variational estimate
\[
    d_\phi(g_\theta(T),z)
    \approx
    \log q_\phi(z\mid g_\theta(T))-\log p(z),
\]
where \(q_\phi(z\mid g)\) is a learned conditional model.

The critic is trained to distinguish positive pairs
\[
    (g,z)\sim p(g,z)
\]
from negative pairs
\[
    (g,z)\sim p(g)p(z).
\]
This training aims to move \(d_\phi\) toward the optimal log-ratio and
thereby tighten the plug-in estimate.

\paragraph{Adversarial training dynamics.}
The optimisation alternates between critic fitting and
encoder updates. The critic parameters
\(\phi\) are trained to maximise the InfoNCE objective, 
approximating the optimal
critic for the current representation.
Then, holding the critic fixed, the encoder
parameters \(\theta\) are trained to minimise the
NCE-CLUB penalty together with the utility loss.
This encourages representations that remain predictive
of $Y$ while discouraging dependence on $Z$.
The gradient-reversed InfoNCE implementation used in our
experiments is described in Section~\ref{sec:implementation}.

\section{Synthetic SCM: experimental details}
\label{appendix:synth}

We generate 1,000 training samples, and 1,000 validation samples from 
the following linear additive SCM with $d = 10$, $d_c = 5$:
\begin{align*}
  Z &\sim \mathcal{N}(0, I_d), \\
  T_C &= Z_{0:d_c} + \epsilon_{T_C},
        & \epsilon_{T_C} &\sim \mathcal{N}(0, I_{d_c}), \\
  T_{nC} &= Z_{d_c:d} + \epsilon_{T_{nC}},
        & \epsilon_{T_{nC}} &\sim \mathcal{N}(0, I_{d-d_c}), \\
  T &= [T_C, T_{nC}], \\
  Y &= \mathbf{1}^\top T_C + \mathbf{1}^\top Z + \epsilon_Y,
        & \epsilon_{Y,i} &\overset{\mathrm{iid}}{\sim}
                            \mathcal{N}(0, \sigma_Y^2).
\end{align*}

The predictor is a linear model
$\hat Y = w_Z^\top Z + w_T^\top T + b$.
We take the treatment representation to be the per-coordinate contribution of $T$ to
the prediction, $g_T = T \odot w_T \in \mathbb{R}^{10}$, where $w_T$
is the treatment slice of the predictor's own weight vector. The MI
critic therefore decorrelates $Z$ from the components of $T$ that the
predictor actually uses. Under gradient reversal this pushes $w_T$
toward the causal coordinates of $T$. The critic itself is bilinear:
two bias-free linear projections $W_g, W_z \in \mathbb{R}^{10 \times 10}$
to a shared embedding space, with InfoNCE scores given by the full
$B \times B$ matrix $S = (W_g g_T)(W_z Z)^\top$ over a minibatch
of size $B$.
We train with lr $3\times10^{-4}$, batch size $64$, $4{,}000$ epochs, and GRL schedule with 
$\lambda = 0$ for epochs 0--200, cubic ramp $0 \to \lambda_{\max} = 0.5$
over epochs 200--1000, held constant thereafter. 

\section{Semi-Synthetic KuaiRand: experimental details}
\label{appendix:kuairand}

We use the KuaiRand-Pure~\cite{gao2022kuairand} (Zenodo \texttt{10439422}) dataset with 
$27{,}285$ users, $7{,}583$ items, and $\sim$1.186M distinct user-video pairs from the randomized-exposure log.
KuaiRand provides $30$ user features, of which we keep $24$ including
four counts, two binary indicators, and $18$ anonymized category codes. We drop the four bucketed versions of the
counts, the categorical activity label, and one binary indicator that is constant in KuaiRand-Pure. For simplicity when generating synthetic outcomes we treat each anonymized category code as numerical.
KuaiRand also provides $62$ video features, of which we keep $54$ including duration, width, height, and $51$ video statistics. We drop identifiers, categorical features, upload date, and constant features. We log transform counts, duration, video statistics, and then
standardize them. We standardize the numerical user features. 

Let $Z$ denote the $24$ user features and $T_C$ the $54$ video
features.
We generate a synthetic click outcome from a user term $h_u(Z)$, video term $h_v(T_C)$, and user-video interaction term $h_{uv}(Z, T_C)$ as
\begin{equation*}
Y\sim\operatorname{Bernoulli}\left(
\operatorname{sigmoid}\left(\operatorname{std}
\left(4\,h_u(Z)+h_v(T_C)+h_{uv}(Z,T_C)\right)+b\right)\right)
\end{equation*}
where $\operatorname{std}$ corresponds to centering and scaling. The intercept $b$ calibrates the mean click probability to the actual click rate in the randomized log.
We define the user, video, and user-video interaction terms as follows. Let $q_1,q_2\in\mathbb{R}^{24}$,
$q_3\in\mathbb{R}^{54}$, $Q_1\in\mathbb{R}^{24\times16}$, and
$Q_2\in\mathbb{R}^{54\times16}$ be fixed random projections. We define
\begin{equation*}
\begin{aligned}
h_u(Z)
  &= \operatorname{std}\left(
       \tanh\bigl(\operatorname{std}(Z^\top q_1)\bigr)
       \tanh\bigl(\operatorname{std}(Z^\top q_2)\bigr)
     \right),\\
h_v(T_C)
  &= \operatorname{std}(T_C^\top q_3),\\
h_{uv}(Z,T_C)
  &= \operatorname{std}\left(
       \frac{1}{16}\sum_{j=1}^{16}
       \tanh\bigl((Z^\top Q_1)_j\bigr)
       \tanh\bigl((T_C^\top Q_2)_j\bigr)
     \right).
\end{aligned}
\end{equation*}
The entries of each projection are independent zero-mean Gaussian samples with
variance equal to the inverse of its input dimension.
To be clear, the data generation process above is not intended to be a realistic model of user behaviour. It provides a nonlinear outcome mechanism with known causal structure built from real KuaiRand features.  

To test our method, we introduce a non-causal treatment component $T_{nC}\in\{1,\ldots,16\}$, which does not impact the outcome, but is related to $Z$ and can be exploited by the model when making predictions. We append this non-causal shortcut feature to the observed treatment input $T=(T_C, T_{nC})$. In the source distribution, $T_{nC}$ is strongly related to the user response term. Let $Q(Z)$ be the quantile bin of $h_u(Z)$, then
\begin{equation*}
p(T_{nC}=t\mid Z)
=\frac{\exp\left(\alpha\mathbf{1}\left[t=Q(Z)\right]\right)}{\exp(\alpha)+15}.
\end{equation*}
The parameter $\alpha$ controls the shortcut strength, and we choose $\alpha$ so that the mutual information is $I(T_{nC};Q\mid V)=2$ nats, where $V$ is the video ID. For the target distribution, we instead set $\alpha=0$, sampling the non-causal feature uniformly and removing the $Z$ shortcut.
Source and target test sets contain the same users, videos, features, and synthetic clicks, only $T_{nC}$ changes. We encode the code as $16$ one-hot values scaled by $3$, so the treatment encoder receives $54+16=70$ inputs.

We train on $70\%$ of the data, use $15\%$ as our validation set for early stopping, and the remaining $15\%$ as the test set for calculating source and target AUC and ranking metrics. 
Our model is a DeepFM~\cite{guo2017deepfm} backbone with $16$-dimensional learned user-ID and video-ID
embeddings. The treatment encoder maps $70\to32\to16$ with a ReLU hidden
layer. Appending the video-ID embedding gives a $32$-dimensional treatment
representation. Appending the user-ID embedding to the $24$ user features
gives a $40$-dimensional user representation. The DeepFM network has hidden
sizes $[32,16]$, dropout $0.2$, and FM embedding dimension $8$.
The CRL variant uses the same predictor and adds an InfoNCE critic.
The predictor architecture and capacity are identical across configurations.
Each critic input is mapped through a two-layer projection with widths $64$ and $32$, followed
by a dot-product score. Gradient reversal acts on the treatment representation.
The user representation is detached from the critic loss. The reversal weight $\lambda$
is zero for $25$ steps, rises linearly to $0.3$ over the next $50$ steps, and
then remains fixed. We train for four epochs with Adam,
learning rate $2\times10^{-3}$, weight decay $10^{-6}$, batch size $4{,}096$,
and gradient-norm clipping at $1$. We evaluate source-validation AUC every
$50$ steps and retain the best checkpoint.

When applying DCCL~\cite{zhao2023dccl} to our KuaiRand setup, we split each $16$-dimensional user and video ID
embedding equally into interest and conformity components,
preserving the predictor's total capacity. Two popularity-weighted
contrastive losses supplement the prediction loss, using clicked examples as anchors and up to $128$ in-batch negatives. Both auxiliary losses use weight $0.01$.

As well as source and target AUC, in Table~\ref{tab:kuairand} we report target Recall@5 and target NDCG@5 
\citep[see eg.][who report similar metrics on KuaiRand]{zhang2023idcf,deng2024ividr}.
For each user in the test split we rank the items
they were exposed to by predicted click probability. Recall@$K$ is
the fraction of that user's positives that
appear in the top-$K$, and NDCG@$K$ uses binary relevance with
discount $1/\log_2(i+1)$ at rank $i$ and IDCG computed over the
user's true positives. Both are averaged over users with at least
one positive and at least five exposed items.


\end{document}